\documentclass[letterpaper, 10 pt, conference]{ieeeconf}  %

\IEEEoverridecommandlockouts                              %

\pdfoutput=1

\usepackage{times}
\usepackage{amssymb,amsfonts,amsmath,amscd}
\usepackage[pdftex,colorlinks]{hyperref}
\usepackage[pdftex]{graphicx}
\usepackage[noadjust]{cite} %
\usepackage{fnpos}
\usepackage{verbatim}
\usepackage{fancyhdr}
\usepackage{url}
\usepackage[nolist]{acronym} %
\usepackage[ruled,vlined,linesnumbered]{algorithm2e}  %
\usepackage{placeins} %
\usepackage{microtype} %
\usepackage{setspace} %

\usepackage[usenames,dvipsnames]{color}
\usepackage{soul}   %

\SetAlgoSkip{}
\newcommand{\newAlgoLine}[1]{ {\color{BrickRed} #1} }	

\DeclareMathOperator*{\argmin}{arg\,min}

\newcommand{\expect}[1]{E\left[#1\right]}
\DeclareMathOperator*{\diag}{diag}

\newcommand{\norm}[2]{\left|\left| #1 \right|\right|_{#2}}

\newcommand{\bbm}{\begin{bmatrix}}
\newcommand{\ebm}{\end{bmatrix}}

\newcommand{\statex}[0]{\mathbf{x}}

\newcommand{\namedState}[1]{\statex_{#1}}
\newcommand{\stateSet}[0]{X}
\newcommand{\namedSet}[1]{\stateSet_{#1}}
\newcommand{\vertexSet}[0]{V}
\newcommand{\edgeSet}[0]{E}
\newcommand{\treeGraph}[0]{\mathcal{T}}
\newcommand{\pathSeq}[0]{\sigma}
\newcommand{\pathSet}[0]{\Sigma}
\newcommand{\cost}[0]{c}
\newcommand{\namedCost}[1]{\cost_{#1}}
\newcommand{\axis}[0]{\mathbf{a}}
\newcommand{\radius}[0]{r}
\newcommand{\fTrue}[1]{f\left(#1\right)}
\newcommand{\gTrue}[1]{g\left(#1\right)}
\newcommand{\hTrue}[1]{h\left(#1\right)}
\newcommand{\fHat}[1]{\widehat{f}\left(#1\right)}
\newcommand{\gHat}[1]{\widehat{g}\left(#1\right)}
\newcommand{\hHat}[1]{\widehat{h}\left(#1\right)}
\newcommand{\uniform}[1]{\mathcal{U}\left(#1\right)}
\newcommand{\lebesgueSymb}[0]{\lambda}
\newcommand{\lebesgue}[1]{\lebesgueSymb\left(#1\right)}

\newcommand{\xstart}[0]{\namedState{\rm start}}
\newcommand{\xrand}[0]{\namedState{\rm rand}}
\newcommand{\xnew}[0]{\namedState{\rm new}}
\newcommand{\xnear}[0]{\namedState{\rm near}}
\newcommand{\xnearest}[0]{\namedState{\rm nearest}}
\newcommand{\xparent}[0]{\namedState{\rm parent}}
\newcommand{\xmin}[0]{\namedState{\rm min}}
\newcommand{\xsoln}[0]{\namedState{\rm soln}}
\newcommand{\xgoal}[0]{\namedState{\rm goal}}
\newcommand{\xto}[0]{\namedState{\rm to}}
\newcommand{\xfrom}[0]{\namedState{\rm from}}

\newcommand{\xcentre}[0]{\namedState{\rm centre}}
\newcommand{\xFirstFoci}[0]{\namedState{f1}}
\newcommand{\xSecondFoci}[0]{\namedState{f2}}
\newcommand{\xball}[0]{\namedState{\rm ball}}
\newcommand{\xellipse}[0]{\namedState{\rm ellipse}}
\newcommand{\xfhat}[0]{\namedState{\widehat{f}}}

\newcommand{\cmin}[0]{\namedCost{\rm min}}
\newcommand{\cnew}[0]{\namedCost{\rm new}}
\newcommand{\cnear}[0]{\namedCost{\rm near}}
\newcommand{\cmax}[0]{\namedCost{\rm max}}
\newcommand{\cbest}[0]{\namedCost{\rm best}}
\newcommand{\cideal}[0]{\cost^{*}}

\newcommand{\obsSet}[0]{\namedSet{\rm obs}}
\newcommand{\freeSet}[0]{\namedSet{\rm free}}
\newcommand{\goalSet}[0]{\namedSet{\rm goal}}
\newcommand{\solnSet}[0]{\namedSet{\rm soln}}
\newcommand{\nearSet}[0]{\namedSet{\rm near}}
\newcommand{\sampleSet}[0]{\namedSet{\rm s}}

\newcommand{\fSet}[0]{\namedSet{f}}
\newcommand{\fhatSet}[0]{\namedSet{\widehat{f}}}

\newcommand{\ballSet}[0]{\namedSet{\rm ball}}
\newcommand{\ellipseSet}[0]{\namedSet{\rm ellipse}}

\newcommand{\bestPath}[0]{\pathSeq^{*}}

\newtheorem{thm}{Theorem}

\newtheorem{rem}{Remark}

\newcommand{\UTIAStitle}{Informed RRT*: Optimal Sampling-based Path Planning Focused via Direct Sampling of an Admissible Ellipsoidal Heuristic}

\title{\LARGE \bf \UTIAStitle}
\author{Jonathan D.\ Gammell$^1$, Siddhartha S.\ Srinivasa$^2$, and Timothy D.\ Barfoot$^1$%
\thanks{$^1$ J.\ D.\ Gammell and T.\ D.\ Barfoot are with the Autonomous Space Robotics Lab at the
University of Toronto Institute for Aerospace Studies, Toronto, Ontario, Canada.
Email: \texttt{\{jon.gammell, tim.barfoot\}@utoronto.ca}}
\thanks{$^2$ S.\ S.\ Srinivasa is with The Robotics Institute, Carnegie Mellon University, Pittsburgh, Pennsylvania, USA. Email: \texttt{siddh@cs.cmu.edu}}
}

\hypersetup{%
    pdftitle={Gammell et al.: \UTIAStitle},
    pdfauthor={Jonathan D. Gammell, Siddhartha S. Srinivasa, and Timothy D. Barfoot},
    pdfkeywords={},
    pdfsubject={},
    pdfstartview=FitH,%
    bookmarksopen=true,%
    breaklinks=true,%
    colorlinks=true,%
    linkcolor=blue,anchorcolor=blue,%
    citecolor=blue,filecolor=blue,%
    menucolor=blue,%
    urlcolor=blue
}%

\newcommand{\algorithmStyle}[0]{\footnotesize} %
\newcommand{\bodySpacing}[0]{0.965} %

\begin{document}
\begin{acronym}[UTIAS]

	\acro{ASRL}{Autonomous Space Robotics Lab}
	\acro{CSA}{Canadian Space Agency}
	\acro{DRDC}{Defence Research and Development Canada}
	\acro{KSR}{Koffler Scientific Reserve at Jokers Hill}
	\acro{MET}{Mars Emulation Terrain}
	\acro{MIT}{Massachusetts Institute of Technology}
	\acro{NASA}{National Aeronautics and Space Administration}
	\acro{NSERC}{Natural Sciences and Engineering Research Council of Canada}
	\acro{NCFRN}{\acs{NSERC} Canadian Field Robotics Network}
	\acro{NORCAT}{Northern Centre for Advanced Technology Inc.}
	\acro{ODG}{Ontario Drive and Gear Ltd.}
	\acro{ONR}{Office of Naval Research}
	\acro{USSR}{Union of Soviet Socialist Republics}
	\acro{UofT}{University of Toronto}
	\acro{UW}{University of Waterloo}
	\acro{UTIAS}{University of Toronto Institute for Aerospace Studies}

	\acro{ACPI}{advanced configuration and power interface}
	\acro{CLI}{command-line interface}
	\acro{GUI}{graphical user interface}
	\acro{LAN}{local area network}
	\acro{MFC}{Microsoft foundation class}
	\acro{NIC}{network interface card}
	\acro{SDK}{software development kit}
	\acro{HDD}{hard-disk drive}
	\acro{SSD}{solid-state drive}

	\acro{IROS}{IEEE/RSJ International Conference on Intelligent Robots and Systems}

	\acro{DOF}{degree-of-freedom}
		\acrodefplural{DoF}[DoFs]{degrees-of-freedom} %

		\acro{FOV}{field of view}
			\acrodefplural{FOV}[FOVs]{fields of view}
		\acro{HDOP}{horizontal dilution of position}
		\acro{UTM}{universal transverse mercator}
		\acro{WAAS}{wide area augmentation system}
		\acro{AHRS}{attitude heading reference system}
		\acro{DAQ}{data acquisition}
		\acro{DGPS}{differential global positioning system}
		\acro{DPDT}{double-pole, double-throw}
		\acro{DPST}{double-pole, single-throw}
		\acro{GPR}{ground penetrating radar}
		\acro{GPS}{global positioning system}
		\acro{LED}{light-emitting diode}
		\acro{IMU}{inertial measurement system}
		\acro{PTU}{pan-tilt unit}
		\acro{RTK}{real-time kinematic}
		\acro{R/C}{radio control}
		\acro{SCADA}{supervisory control and data acquisition}
		\acro{SPST}{single-pole, single-throw}
		\acro{SPDT}{single-pole, double-throw}
		\acro{UWB}{ultra-wide band}

	\acro{DDS}{Departmental Doctoral Seminar}
	\acro{DEC}{Doctoral Examination Committee}
	\acro{FOE}{Final Oral Exam}
	\acro{ICD}{Interface Control Document}

	\acro{iid}[i.i.d.]{independent and identically distributed}

	\acro{EKF}{extended Kalman filter}
	\acro{iSAM}{incremental smoothing and mapping}
	\acro{ISRU}{in-situ resource utilization}
	\acro{PCA}{principle component analysis}
	\acro{SLAM}{simultaneous localization and mapping}
	\acro{SVD}{singular value decomposition}
	\acro{UKF}{unscented Kalman filter}
	\acro{VO}{visual odometry}
	\acro{VTR}[VT\&R]{visual teach and repeat}
		\acro{BITstar}[BIT*]{Batch Informed Trees}
		\acro{BRM}{belief roadmap}
		\acro{EST}{Expansive Space Tree}
    	\acro{FMT}[FMT*]{fast marching tree}
		\acro{LQG-MP}{linear-quadratic Gaussian motion planning}
		\acro{LPAstar}[LPA*]{lifelong planning A*}
		\acro{MDP}{Markov decision process}
			\acrodefplural{MDP}[MDPs]{Markov decision processes}
		\acro{NRP}{network of reusable paths}
			\acrodefplural{NRP}[NRPs]{networks of reusable paths}
		\acro{POMDP}{partially-observable Markov decision process}
			\acrodefplural{POMDP}[POMDPs]{partially-observable Markov decision processes}
		\acro{PRM}{Probabilistic Roadmap}		
		\acro{PRMstar}[PRM*]{optimal \acp{PRM}}
		\acro{RRG}{Rapidly-exploring Random Graph}
		\acro{RRM}{Rapidly-exploring Roadmap}		
		\acro{RRT}{Rapidly-exploring Random Tree}
			\acro{hRRT}{Heuristically Guided \acs{RRT}}
		\acro{RRTstar}[RRT*]{optimal \acs{RRT}}
		\acro{RRTehstar}[RRTeh*]{\ac{RRTstar} with ellipsoidal heuristics}
		\acro{RRBT}{rapidly-exploring random belief tree}

	\acro{MER}{Mars Exploration Rover}
	\acro{MSL}{Mars Science Laboratory}
    \acro{OMPL}{open motion planning library}
	\acro{ROS}{Robot Operating System}

\end{acronym} %

\maketitle
\thispagestyle{empty}
\pagestyle{empty}

\setlength\parskip{0ex plus0.1ex minus0.1ex}

\makeatletter
\renewcommand\section{\@startsection{section}{1}{\z@}%
{1.5ex plus 1ex minus 0ex}%
{0.7ex plus 0.5ex minus 0ex}%
{\normalfont\normalsize\centering\scshape}}%
\makeatother

\setlength \abovedisplayskip{1ex plus0pt minus1pt}
\setlength \belowdisplayskip{1ex plus0pt minus1pt}

\setlength{\skip\footins}{0.5\baselineskip  plus 0.0\baselineskip  minus 0.2\baselineskip} %

\setlength\floatsep{0.9\baselineskip plus0pt minus0.2\baselineskip}                     %
\setlength\textfloatsep{0.2\baselineskip plus0pt minus0.4\baselineskip}                 %
\setlength\abovecaptionskip{0.pt plus0pt minus0pt}                                      %

\begin{spacing}{\bodySpacing}%
\begin{abstract}%
Rapidly-exploring random trees (\acsp{RRT}) are popular in motion planning because they find solutions efficiently to single-query problems.\acused{RRT}
Optimal \acsp{RRT} (\acsp{RRTstar}) extend \acp{RRT} to the problem of finding the optimal solution, but in doing so asymptotically find the optimal path from the initial state to \emph{every} state in the planning domain.\acused{RRTstar}
This behaviour is not only inefficient but also inconsistent with their single-query nature.

For problems seeking to minimize path length, the subset of states that can improve a solution can be described by a prolate hyperspheroid.
We show that unless this subset is sampled directly, the probability of improving a solution becomes arbitrarily small in large worlds or high state dimensions.
In this paper, we present an exact method to focus the search by directly sampling this subset.

The advantages of the presented sampling technique are demonstrated with a new algorithm, \emph{Informed} \acs{RRTstar}.
This method retains the same probabilistic guarantees on completeness and optimality as \ac{RRTstar} while improving the convergence rate and final solution quality.
We present the algorithm as a simple modification to \ac{RRTstar} that could be further extended by more advanced path-planning algorithms.
We show experimentally that it outperforms \ac{RRTstar} in rate of convergence, final solution cost, and ability to find difficult passages while demonstrating less dependence on the state dimension and range of the planning problem.
\end{abstract}%
\acresetall %

\section{Introduction}\label{sec:intro}%
The motion-planning problem is commonly solved by first discretizing the continuous state space with either a grid for graph-based searches or through random sampling for stochastic incremental searches.
Graph-based searches, such as A* \cite{hart_tssc68}, are often \emph{resolution complete} and \emph{resolution optimal}.
They are guaranteed to find the optimal solution, if a solution exists, and return failure otherwise (up to the resolution of the discretization).
These graph-based algorithms do not scale well with problem size (e.g., state dimension or problem range).%

Stochastic searches, such as \acp{RRT} \cite{lavalle_ijrr01}, \acp{PRM} \cite{kavraki_tro96}, and \acp{EST} \cite{hsu_ijrr02}, use sampling-based methods to avoid requiring a discretization of the state space.
This allows them to scale more effectively with problem size and to directly consider kinodynamic constraints; however, the result is a less-strict completeness guarantee.
\acp{RRT} are \emph{probabilistically complete}, guaranteeing that the probability of finding a solution, if one exists, approaches unity as the number of iterations approaches infinity.

Until recently, these sampling-based algorithms made no claims about the optimality of the solution.
Urmson and Simmons \cite{urmson_iros03} had found that using a heuristic to bias sampling improved \ac{RRT} solutions, but did not formally quantify the effects.
Ferguson and Stentz \cite{ferguson_iros06} recognized that the length of a solution bounds the possible improvements from above, and demonstrated an iterative anytime \ac{RRT} method to solve a series of subsequently smaller planning problems.
Karaman and Frazzoli \cite{karaman_ijrr11} later showed that \acp{RRT} return a suboptimal path with probability one, demonstrating that all \acs{RRT}-based methods will almost surely be suboptimal and presented a new class of optimal planners.
They named their optimal variants of \acp{RRT} and \acp{PRM}, \acs{RRTstar} and \acs{PRMstar}, respectively.\acused{RRTstar}\acused{PRMstar} 
These algorithms are shown to be \emph{asymptotically optimal}, with the probability of finding the optimal solution approaching unity as the number of iterations approaches infinity.

\begin{figure}[t]%
	\centering
	\includegraphics[width=\columnwidth]{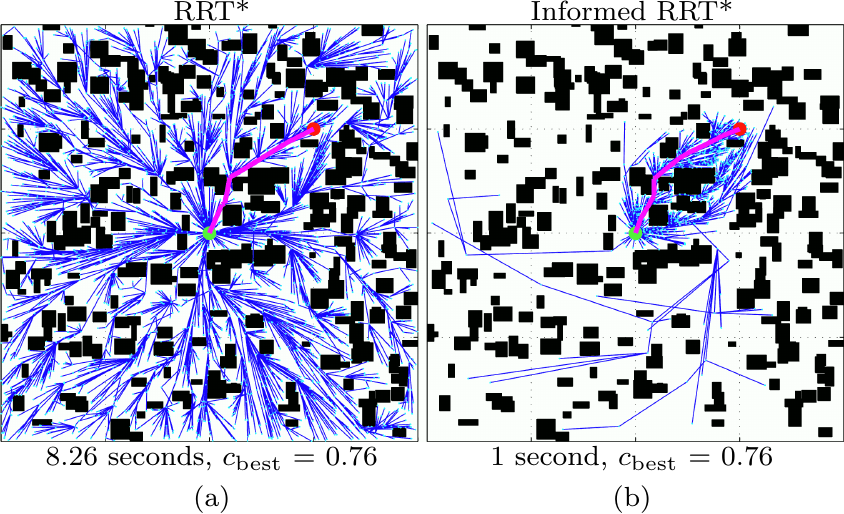}
	\caption{Solutions of equivalent cost found by \acs{RRTstar} and Informed \acs{RRTstar} on a random world. After an initial solution is found, Informed \acs{RRTstar} focuses the search on an ellipsoidal informed subset of the state space, $\fhatSet \subseteq \stateSet$, that contains all the states that can improve the current solution regardless of homotopy class. This allows Informed \acs{RRTstar} to find a better solution faster than \acs{RRTstar} without requiring any additional user-tuned parameters.}
	\label{fig:randomWorld2}
\end{figure}%
\begin{figure*}[t]%
	\centering
	\includegraphics[width=\textwidth]{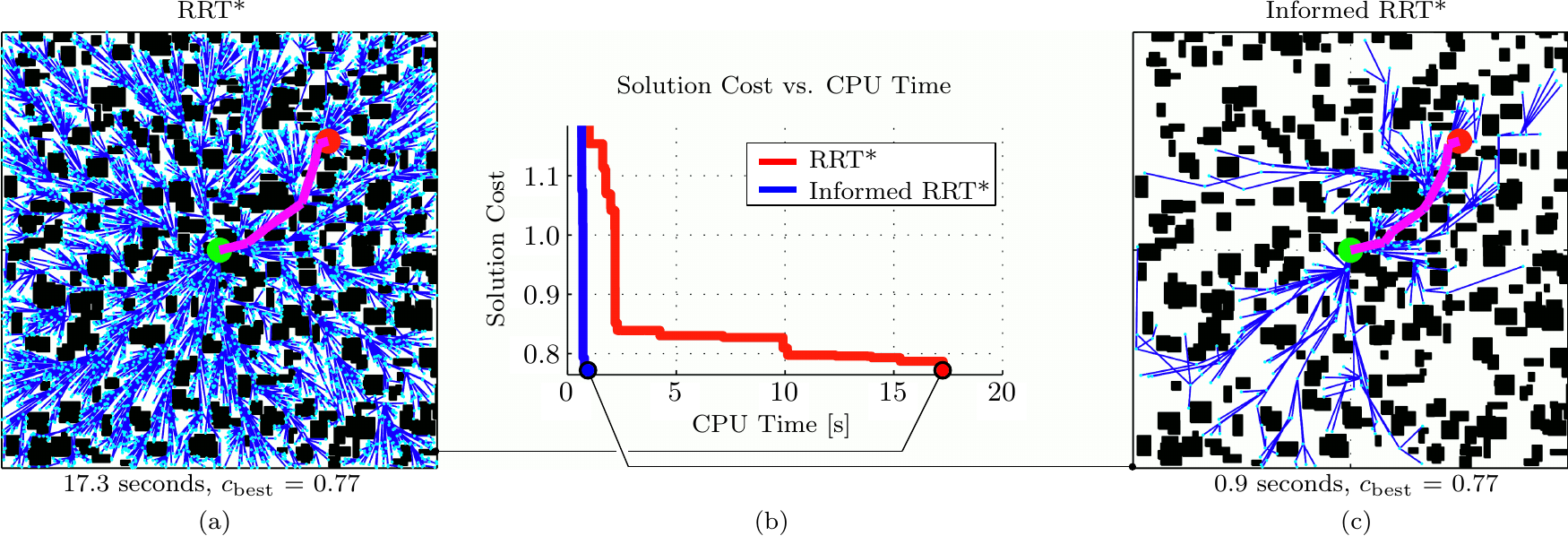}
	\caption{The solution cost versus computational time for \acs{RRTstar} and Informed \acs{RRTstar} on a random world problem. Both planners were run until they found a solution of the same cost. Figs.~(a, c) show the final result, while Fig.~(b) shows the solution cost versus computational time. From Fig.~(a), it can be observed that \acs{RRTstar} spends significant computational resources exploring regions of the planning problem that cannot possibly improve the current solution, while Fig.~(c) demonstrates how Informed \acs{RRTstar} focuses the search.
    \vspace{-1.5\baselineskip}. }

	\label{fig:randomWorld}
\end{figure*}%

\acp{RRT} are not asymptotically optimal because the existing state graph biases future expansion.
\ac{RRTstar} overcomes this by introducing incremental rewiring of the graph.
New states are not only added to a tree, but also considered as replacement parents for existing nearby states in the tree.
With uniform global sampling, this results in an algorithm that asymptotically finds the optimal solution to the planning problem by \emph{asymptotically finding the optimal paths from the initial state to every state in the problem domain}.
This is inconsistent with their single-query nature and becomes expensive in high dimensions.

In this paper, we present the \emph{focused} optimal planning problem as it relates to the minimization of path length in $\mathbb{R}^n$.
For such problems, a necessary condition to improve the solution at any iteration is the addition of states from an ellipsoidal subset of the planning domain \cite{ferguson_iros06}, \cite{gabriely_tro08,gasilov_iete11,otte_tro13}.
We show that the probability of adding such states through uniform sampling becomes arbitrarily small as the size of the planning problem increases or the solution approaches the theoretical minimum, and present an exact method to sample the ellipsoidal subset directly.
It is also shown that with strict assumptions (i.e., no obstacles) that this direct sampling results in linear convergence to the optimal solution.

This direct-sampling method allows for the creation of informed-sampling planners.
Such a planner, Informed \acs{RRTstar}, is presented to demonstrate the advantages of \emph{informed} incremental search (Fig.~\ref{fig:randomWorld2}).
Informed \acs{RRTstar} behaves as \ac{RRTstar} until a first solution is found, after which it only samples from the subset of states defined by an admissible heuristic to possibly improve the solution.
This subset implicitly balances exploitation versus exploration and requires no additional tuning (i.e., there are no additional parameters) or assumptions (i.e., all relevant homotopy classes are searched).
While heuristics may not always improve the search, their prominence in real-world planning demonstrates their practicality.
In situations where they provide no additional information (e.g., when the informed subset includes the entire planning problem), Informed \acs{RRTstar} is equivalent to \ac{RRTstar}.

Informed \acs{RRTstar} is a simple modification to \ac{RRTstar} that demonstrates a clear improvement.
In simulation, it performs as well as existing \ac{RRTstar} algorithms on simple configurations, and demonstrates order-of-magnitude improvements as the configurations become more difficult (Fig.~\ref{fig:randomWorld}).
As a result of its focused search, the algorithm has less dependence on the dimension and domain of the planning problem as well as the ability to find better topologically distinct paths sooner.
It is also capable of finding solutions within tighter tolerances of the optimum than \ac{RRTstar} with equivalent computation, and in the absence of obstacles can find the optimal solution to within machine zero in finite time (Fig.~\ref{fig:machineZero}).
It could also be used in combination with other algorithms, such as path-smoothing, to further reduce the search space.

The remainder of this paper is organized as follows.
Section~\ref{sec:back} presents a formal definition of the focused optimal planning problem and reviews the existing literature.
Section~\ref{sec:ellipse} presents a closed-form estimate of the subset of states that can improve a solution for problems seeking to minimize path length in $\mathbb{R}^n$ and analyzes the implications on \ac{RRTstar}-style algorithms.
Section~\ref{sec:sample} presents a method to sample this subset directly.
Section~\ref{sec:algorithm} presents the Informed \acs{RRTstar} algorithm and Section~\ref{sec:sim} presents simulation results comparing \ac{RRTstar} and Informed \acs{RRTstar} on simple planning problems of various size and configuration and random problems of various dimension.
Section~\ref{sec:end} concludes the paper with a discussion of the technique and some related ongoing work.

\begin{figure*}[t]%
	\centering
	\includegraphics[scale=1]{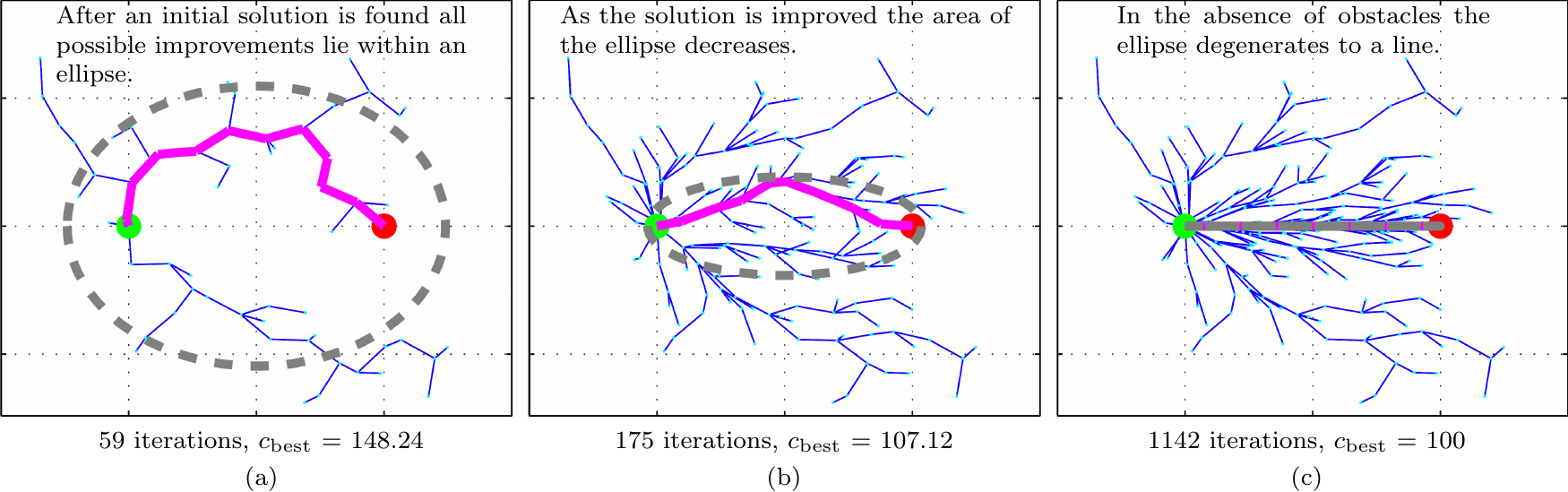}
	\caption{Informed \acs{RRTstar} converging to within machine zero of the optimum in the absence of obstacles. The start and goal states are shown as green and red, respectively, and are 100 units apart. The current solution is highlighted in magenta, and the ellipsoidal sampling domain, $\fhatSet$, is shown as a grey dashed line for illustration. Improving the solution decreases the size of the sampling domain, creating a feedback effect that converges to within machine zero of the theoretical minimum. Fig.~(a) shows the first solution at 59 iterations, (b) after 175 iterations, and (c), the final solution after 1142 iterations, at which point the ellipse has degenerated to a line between the start and goal.\vspace{-1.5\baselineskip}}
	\label{fig:machineZero}
\end{figure*}%
\section{Background}\label{sec:back}%
\subsection{Problem Definition}\label{sec:back:defn}%
We define the optimal planning problem similarly to \cite{karaman_ijrr11}.
Let $\stateSet \subseteq \mathbb{R}^n$ be the state space of the planning problem. Let $\obsSet \subsetneq \stateSet$ be the states in collision with obstacles and $\freeSet = \stateSet \setminus \obsSet$ be the resulting set of permissible states.
Let $\xstart \in \freeSet$ be the initial state and $\xgoal \in \freeSet$ be the desired final state.
Let $\pathSeq : \; \left[0,1\right] \mapsto \stateSet$ be a sequence of states (a path) and $\pathSet$ be the set of all nontrivial paths.

The optimal planning problem is then formally defined as the search for the path, $\bestPath$, that minimizes a given cost function, $\cost : \; \pathSet \mapsto \mathbb{R}_{\geq0}$, while connecting $\xstart$ to $\xgoal$ through free space,
\begin{align*}
\bestPath = \argmin_{\pathSeq \in \pathSet}\left\lbrace \cost\left(\pathSeq \right) \;\; \middle| \;\; \right. & \pathSeq(0) = \xstart,\, \pathSeq(1) = \xgoal,\\
& \;\;\; \left. \forall s \in \left[ 0,1 \right],\, \pathSeq\left(s\right) \in \freeSet \right\rbrace,
\end{align*}%
where $\mathbb{R}_{\geq0}$ is the set of non-negative real numbers.

Let $\fTrue{\statex}$ be the cost of an optimal path from $\xstart$ to $\xgoal$ constrained to pass through $\statex$.
Then the subset of states that can improve the current solution, $\fSet \subseteq \stateSet$, can be expressed in terms of the current solution cost, $\cbest$,
\begin{align}\label{eqn:fset}
	\fSet = \left\lbrace \statex \in \stateSet \;\; \middle| \;\; \fTrue{\statex} < \cbest \right\rbrace.
\end{align}%
The problem of focusing \ac{RRTstar}'s search in order to increase the convergence rate is equivalent to increasing the probability of adding a random state from $\fSet$.

As $\fTrue{\cdot}$ is generally unknown, a heuristic function, $\fHat{\cdot}$, may be used as an estimate.
This heuristic is referred to as \emph{admissible} if it never overestimates the true cost of the path, i.e., $\forall \statex \in \stateSet, \; \fHat{\statex} \leq \fTrue{\statex}$.
An estimate of $\fSet$, $\fhatSet$, can then be defined analogously to \eqref{eqn:fset}.
For admissible heuristics, this estimate is guaranteed to completely contain the true set, $\fhatSet \supseteq \fSet$, and thus inclusion in the estimated set is also a necessary condition to improving the current solution.

\subsection{Prior Work}\label{sec:back:lit}%
Prior work to focus \ac{RRT} and \ac{RRTstar} has relied on sample biasing, heuristic-based sample rejection, heuristic-based graph pruning, and/or iterative searches.

\subsubsection{Sample Biasing}\label{sec:back:list:bias}%
Sample biasing attempts to increase the frequency that states are sampled from $\fSet$ by biasing the distribution of samples drawn from $\stateSet$.
This continues to add states from outside of $\fSet$ that cannot improve the solution.
It also results in a nonuniform density over the problem being searched, violating a key \ac{RRTstar} assumption.

\paragraph{Heuristic-biased Sampling}\label{sec:back:lit:bias:f}%
Heuristic-biased sampling attempts to increase the probability of sampling $\fSet$ by weighting the sampling of $\stateSet$ with a heuristic estimate of each state. %
It is used to improve the quality of a regular \ac{RRT} by Urmson and Simmons \cite{urmson_iros03} in the \ac{hRRT} by selecting states with a probability inversely proportional to their heuristic cost.
The \ac{hRRT} was shown to find better solutions than \ac{RRT}; however, the use of RRTs means that the solution is almost surely suboptimal \cite{karaman_ijrr11}.

Kiesel et al. \cite{kiesel_socs12} use a two-stage process to create an \ac{RRTstar} heuristic in their \emph{f-biasing} technique.
A coarse abstraction of the planning problem is initially solved to provide a heuristic cost for each discrete state.
\ac{RRTstar} then samples new states by randomly selecting a discrete state and sampling inside it with a continuous uniform distribution.
The discrete sampling is biased such that states belonging to the abstracted solution have the highest probability of selection.
This technique provides a heuristic bias for the full duration of the \ac{RRTstar} algorithm; however, to account for the discrete abstraction it maintains a nonzero probability of selecting every state.
As a result, states that cannot improve the current solution are still sampled.

\paragraph{Path Biasing}\label{sec:back:lit:bias:path}%
Path-biased sampling attempts to increase the frequency of sampling $\fSet$ by sampling around the current solution path.
This assumes that the current solution is either homotopic to the optimum or separated only by small obstacles.
As this assumption is not generally true, path-biasing algorithms must also continue to sample globally to avoid local optima.
The ratio of these two sampling methods is frequently a user-tuned parameter.

Alterovitz et al. \cite{alterovitz_icra11} use path biasing to develop the \ac{RRM}.
Once an initial solution is found, each iteration of the \ac{RRM} either samples a new state or selects an existing state from the current solution and refines it.
Path refinement occurs by connecting the selected state to its neighbours resulting in a graph instead of a tree.

Akgun and Stilman \cite{akgun_iros11} use path biasing in their dual-tree version of \ac{RRTstar}.
Once an initial solution is found, the algorithm spends a user-specified percentage of its iterations refining the current solution.
It does this by randomly selecting a state from the solution path and then explicitly sampling from its Voronoi region.
This increases the probability of improving the current path at the expense of exploring other homotopy classes.
Their algorithm also employs sample rejection in exploring the state space (Section~\ref{sec:back:reject}).

Nasir et al. \cite{nasir_ijars13} combine path biasing with smoothing in their \acs{RRTstar}-Smart algorithm.
When a solution is found, \acs{RRTstar}-Smart first smooths and reduces the path to its minimum number of states before using these states as biases for further sampling.
This adds the complexity of a path-smoothing algorithm to the planner while still requiring global sampling to avoid local optima.
While the path smoothing quickly reduces the cost of the current solution, it may also reduce the probability of finding a different homotopy class by removing the number of bias points about which samples are drawn and further violates the \ac{RRTstar} assumption of uniform density.

Kim et al. \cite{kim_icra14} use a visibility analysis to generate an initial bias in their Cloud RRT* algorithm.
This bias is updated~as a solution is found to further concentrate sampling near the~path.

\subsubsection{Heuristic-based Sample Rejection}\label{sec:back:reject}%
Heuristic-based sample rejection attempts to increase the real-time rate of sampling $\fSet$ by using rejection sampling on $\stateSet$ to sample $\fhatSet$.
Samples drawn from a larger distribution are either kept or rejected based on their heuristic value.
Akgun and Stilman \cite{akgun_iros11} use such a technique in their algorithm.
While this is computationally inexpensive for a single iteration, the number of iterations necessary to find a single state in $\fhatSet$ is proportional to its size relative to the sampling domain.
This becomes nontrivial as the solution approaches the theoretical minimum or the planning domain grows.

Otte and Correll \cite{otte_tro13} draw samples from a subset of the planning domain in their parallelized C-FOREST algorithm.
This subset is defined as the hyperrectangle that bounds the prolate hyperspheroidal informed subset.
While this improves the performance of sample rejection, its utility decreases as the dimension of the problem increases (Remark~\ref{rem:rect}).

\subsubsection{Graph Pruning}\label{sec:back:prune}%
Graph pruning attempts to increase the real-time exploration of $\fSet$ by using a heuristic function to limit the graph to $\fhatSet$.
States in the planning graph with a heuristic cost greater than the current solution are periodically removed while global sampling is continued.
The space-filling nature of \acp{RRT} biases the expansion of the pruned graph towards the perimeter of $\fhatSet$.
After the subset is filled, only samples from within $\fhatSet$ itself can add new states to the graph.
In this way, graph pruning becomes a rejection-sampling method after greedily filling the target subset.
As adding a new state to an \ac{RRT} requires a call to a nearest-neighbour algorithm, graph pruning will be more computationally expensive than simple sample rejection while still suffering from the same probabilistic limitations.

Karaman et al. \cite{karaman_icra11} use graph pruning to implement an anytime version of \ac{RRTstar} that improves solutions during execution.
They use the current vertex cost plus a heuristic estimate of the cost from the vertex to the goal to periodically remove states from the tree that cannot improve the current solution.
As \ac{RRTstar} asymptotically approaches the optimal cost of a vertex \emph{from above}, this is an inadmissible heuristic for the cost of a solution through a vertex (Section~\ref{sec:ellipse}).
This can overestimate the heuristic cost of a vertex resulting in erroneous removal, especially early in the algorithm when the tree is coarse.
Jordan and Perez \cite{jordan_tech13} use the same inadmissible heuristic in their bidirectional \ac{RRTstar} algorithm.

Arslan and Tsiotras \cite{arslan_icra13} use a graph structure and \ac{LPAstar} \cite{koenig_ai04} techniques in the \acs{RRT}\# algorithm to prune the existing graph.
Each existing state is given a \ac{LPAstar}-style key that is updated after the addition of each new state.
Only keys that are less than the current best solution are updated, and only up-to-date keys are available for connection with newly drawn samples.

\subsubsection{Anytime \acsp{RRT}}%
Ferguson and Stentz \cite{ferguson_iros06} recognized that a solution bounds the subset of states that can provide further improvement from above.
Their iterative \ac{RRT} method, Anytime \acsp{RRT}, solves a series of independent planning problems whose domains are defined by the previous solution.
They represent these domains as ellipses [\citen{ferguson_iros06}, Fig.~2], but do not discuss how to generate samples.
Restricting the planning domain encourages each \ac{RRT} to find a better solution than the previous; however, to do so they must discard the states already found in $\fhatSet$.

\vspace*{1.5ex}
The algorithm presented in this paper calculates $\fhatSet$ explicitly and samples from it directly.
Unlike path biasing it makes no assumptions about the homotopy class of the optimum and unlike heuristic biasing does not explore states that cannot improve the solution.
As it is based on \ac{RRTstar}, it is able to keep all states found in $\fhatSet$ for the duration of the search, unlike Anytime \acsp{RRT}.
By sampling $\fhatSet$ directly, it always samples potential improvements regardless of the relative size of $\fhatSet$ to $\stateSet$.
This allows it to work effectively regardless of the size of the planning problem or the relative cost of the current solution to the theoretical minimum, unlike sample rejection and graph pruning methods.
In problems where the heuristic does not provide any additional information, it performs identically to \ac{RRTstar}.

\begin{figure}[tb]%
	\centering
	\includegraphics[scale=1.0]{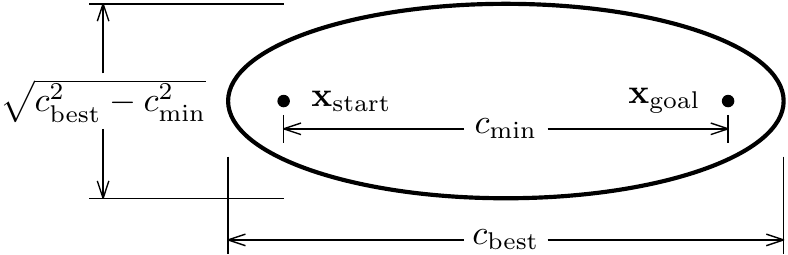}
	\caption{The heuristic sampling domain, $\fhatSet$, for a $\mathbb{R}^2$ problem seeking to minimize path length is an ellipse with the initial state, $\xstart$, and the goal state, $\xgoal$ as focal points. The shape of the ellipse depends on both the initial and goal states, the theoretical minimum cost between the two, $\cmin$, and the cost of the best solution found to date, $\cbest$. The eccentricity of the ellipse is given by $\cmin/\cbest$.}
	\label{fig:ellipse}
\end{figure}%
\section{Analysis of the Ellipsoidal Informed Subset}\label{sec:ellipse}%
Given a positive cost function, the cost of an optimal path from $\xstart$ to $\xgoal$ constrained to pass through $\statex \in \stateSet$, $\fTrue{\statex}$, is equal to the cost of the optimal path from $\xstart$ to $\statex$, $\gTrue{\statex}$, plus the cost of the optimal path from $\statex$ to $\xgoal$, $\hTrue{\statex}$.
As \ac{RRTstar}-based algorithms asymptotically approach the optimal path to every state \emph{from above}, an admissible heuristic estimate, $\fHat{\cdot}$, must estimate both these terms.
A sufficient condition for admissibility is that the components, $\gHat{\cdot}$ and $\hHat{\cdot}$, are individually admissible heuristics of $\gTrue{\cdot}$ and $\hTrue{\cdot}$, respectively.

For problems seeking to minimize path length in $\mathbb{R}^n$, Euclidean distance is an admissible heuristic for both terms (even with motion constraints).
This \emph{informed} subset of states that may improve the current solution, $\fhatSet \supseteq \fSet$, can then be expressed in closed form in terms of the cost of the current solution, $\cbest$, as
\begin{align*}%
\fhatSet = \left\lbrace \statex \in \stateSet \;\; \middle| \;\; \norm{\xstart - \statex}{2} + \norm{\statex - \xgoal}{2} \leq \cbest \right\rbrace,
\end{align*}%
which is the general equation of an $n$-dimensional prolate hyperspheroid (i.e., a special hyperellipsoid).
The focal points are $\xstart$ and $\xgoal$, the transverse diameter is $\cbest$, and the conjugate diameters are $\sqrt{\cbest^2 - \cmin^2}$ (Fig.~\ref{fig:ellipse}).

Admissibility of $\fHat{\cdot}$ makes adding a state in $\fhatSet$ a necessary condition to improve the solution.
With the space-filling nature of \ac{RRT}, the probability of adding such a state quickly becomes the probability of sampling such a state\footnotemark{}.
\footnotetext{States may be added to $\fhatSet$ with a sample from outside the subset until it is filled to within the \ac{RRT} growth-limiting parameter, $\eta$, of its boundary.}%
Thus, the probability of improving the solution at any iteration by uniformly sampling a larger subset, $\statex^{i+1} \sim \uniform{\sampleSet},\, \sampleSet \supseteq \fhatSet$, is less than or equal to the ratio of set measures $\lebesgue{\cdot}$,
\begin{align}\label{eqn:sampleProb}%
    P\left(\cbest^{i+1}<\cbest^i\right) &\leq P\left(\statex^{i+1}\in\fSet \right)\\
        &\leq P\left(\statex^{i+1}\in\fhatSet \right) = \tfrac{\lebesgue{\fhatSet}}{\lebesgue{\sampleSet}}.\nonumber
\end{align}%
Using the volume of a prolate hyperspheroid in $\mathbb{R}^n$ gives
\begin{align}\label{eqn:ellipseProb}%
	P\left(\cbest^{i+1}<\cbest^i\right) \leq \tfrac{\cbest^i \left( \cbest^{i^2} - \cmin^2 \right)^{\tfrac{n-1}{2}} \zeta_n}{2^n\lebesgue{\sampleSet}},
\end{align}%
with $\zeta_n$ being the volume of a unit $n$-ball.

\vspace{0.66ex}\begin{rem}[Rejection sampling]\label{rem:rej}%
From \eqref{eqn:ellipseProb} it can be observed that the probability of improving a solution through uniform sampling becomes arbitrarily small for large subsets (e.g., global sampling) or as the solution approaches the theoretical minimum.
\end{rem}%

\vspace{0.66ex}\begin{rem}[Rectangular rejection sampling]\label{rem:rect}%
Let $\sampleSet$ be a hyperrectangle that tightly bounds the informed subset (i.e., the widths of each side correspond to the diameters of the prolate hyperspheroid) \cite{otte_tro13}.
From \eqref{eqn:ellipseProb}, the probability that a sample drawn uniformly from $\sampleSet$ will be in $\fhatSet$ is then $\frac{\zeta_n}{2^n}$, which decreases rapidly with $n$.
For example, with $n=6$ this gives a maximum $8\%$ probability of improving a solution at each iteration through rejection sampling regardless of the specific solution, problem, or algorithm parameters.
\end{rem}

\vspace{0.66ex}\begin{thm}[Obstacle-free linear convergence]\label{thm:converge}%
With uniform sampling of the informed subset, $\statex \sim \uniform{\fhatSet}$, the cost of the best solution, $\cbest$, converges linearly to the theoretical minimum, $\cmin$, in the absence of obstacles.  
\end{thm}%
\begin{proof}%
The heuristic value of a state is equal to the transverse diameter of a prolate hyperspheroid that passes through the state and has focal points at $\xstart$ and $\xgoal$.
With uniform sampling, the expectation is then \cite{gammell_arxiv14}
\begin{align} \label{eqn:expect}%
    \expect{\fHat{\statex}} = \tfrac{n\cbest^2 + \cmin^2}{\left(n+1\right)\cbest}.
\end{align}%
We assume that the \ac{RRTstar} rewiring parameter is greater than the diameter of the informed subset, similarly to how the proof of the asymptotic optimality of \ac{RRTstar} assumes that $\eta$ is greater than the diameter of the planning problem \cite{karaman_ijrr11}.
The expectation of the solution cost, $\cbest^{i}$, is then the expectation of the heuristic cost of a sample drawn from a prolate hyperspheroid of diameter $\cbest^{i-1}$, i.e., $\expect{\cbest^{i}} = \expect{\fHat{\statex^i}}$.
From \eqref{eqn:expect} it follows that the solution cost converges linearly with a rate, $\mu$, that depends only on the state dimension \cite{gammell_arxiv14},
\setlength \belowdisplayskip {-2ex} %
\begin{align*}%
    \mu =  \left. \tfrac{\partial \expect{\cbest^{i}}}{\partial \cbest^{i-1}} \right|_{\cbest^{i-1} = \cmin} = \tfrac{n - 1}{n+1}.
\end{align*}%
\setlength \belowdisplayskip {1ex plus0pt minus1pt} %
\end{proof}%
While the obstacle-free assumption is impractical, Thm.~\ref{thm:converge} illustrates the fundamental effectiveness of direct informed sampling and provides possible insight for future work.

\section{Direct Sampling of an Ellipsoidal Subset}\label{sec:sample}
Uniformly distributed samples in a hyperellipsoid, $\xellipse \sim \uniform{\ellipseSet}$, can be generated by transforming uniformly distributed samples from the unit $n$-ball, $\xball \sim \uniform{\ballSet}$,
\begin{align*}%
	\xellipse = \mathbf{L} \xball + \xcentre,
\end{align*}%
where $\xcentre = \left(\xFirstFoci + \xSecondFoci\right)/2$ is the centre of the hyperellipsoid in terms of its two focal points, $\xFirstFoci$ and $\xSecondFoci$, and $\ballSet = \left\lbrace \statex \in \stateSet \;\; \middle| \;\; \norm{\statex}{2} \leq 1 \right\rbrace$ \cite{sun_fusion02}.

This transformation can be calculated by Cholesky decomposition of the hyperellipsoid matrix, $\mathbf{S} \in \mathbb{R}^{n \times n}$,
\setlength \belowdisplayskip{-0.5ex plus0pt minus0pt}%
\begin{align*}%
	\mathbf{L}\mathbf{L}^T \equiv \mathbf{S},
\end{align*}%
\setlength \belowdisplayskip{1ex plus0pt minus1pt}%
where
\setlength \abovedisplayskip{-0.5ex plus0pt minus0pt}%
\begin{align*}%
	\left( \statex - \xcentre \right)^T\mathbf{S}\left( \statex - \xcentre \right) = 1,
\end{align*}%
\setlength \abovedisplayskip{1ex plus0pt minus1pt}%
with $\mathbf{S}$ having eigenvectors corresponding to the axes of the hyperellipsoid, $\left\lbrace \axis_i \right\rbrace$, and eigenvalues corresponding to the squares of its radii, $\left\lbrace \radius_i^2 \right\rbrace$.
The transformation, $\mathbf{L}$, maintains the uniform distribution in $\ellipseSet$ \cite{gammell_arxiv14b}.

For prolate hyperspheroids, such as $\fhatSet$, the transformation can be calculated from just the transverse axis and the radii.
The hyperellipsoid matrix in a coordinate system aligned with the transverse axis is the diagonal matrix
\begin{align*}%
	\mathbf{S} = \diag \left\lbrace \tfrac{\cbest^2}{4}, \tfrac{\cbest^2 - \cmin^2}{4}, \ldots, \tfrac{\cbest^2 - \cmin^2}{4} \right\rbrace,
\end{align*}%
with a resulting decomposition of
\begin{align}\label{eqn:finalL}%
	\mathbf{L} = \diag \left\lbrace \tfrac{\cbest}{2}, \tfrac{\sqrt{\cbest^2 - \cmin^2}}{2}, \ldots, \tfrac{\sqrt{\cbest^2 - \cmin^2}}{2} \right\rbrace,
\end{align}%
where $\diag\left\lbrace\cdot\right\rbrace$ denotes a diagonal matrix.

The rotation from the hyperellipsoid frame to the world frame, $\mathbf{C} \in SO\left( n \right)$, can be solved directly as a general Wahba problem \cite{wahba_siam65}.
It has been shown that a valid solution can be found even when the problem is underspecified \cite{ruiter_jgcs13}.
The rotation matrix is given by
\begin{align} \label{eqn:svd}%
	\mathbf{C} = \mathbf{U}\diag\left\lbrace 1, \ldots, 1, \det\left(\mathbf{U}\right) \det\left(\mathbf{V}\right) \right\rbrace \mathbf{V}^{T},
\end{align}%
where $\det\left(\cdot\right)$ is the matrix determinant and $\mathbf{U} \in \mathbb{R}^{n\times n}$ and $\mathbf{V} \in \mathbb{R}^{n\times n}$ are unitary matrices such that $\mathbf{U}\boldsymbol{\Sigma}\mathbf{V}^T \equiv \mathbf{M}$ via singular value decomposition.
The matrix $\mathbf{M}$ is given by the outer product of the transverse axis in the world frame, $\axis_1$, and the first column of the identity matrix, $\mathbf{1}_1$,
\setlength \belowdisplayskip{-0.5ex plus0pt minus0pt}%
\begin{align*}%
	\mathbf{M} = \axis_1\mathbf{1}_1^T,
\end{align*}%
\setlength \belowdisplayskip{1ex plus0pt minus1pt}%
where
\setlength \abovedisplayskip{-0.5ex plus0pt minus0pt}%
\begin{align*}%
	\axis_{1} = \left( \xgoal - \xstart \right)/\norm{\xgoal - \xstart}{2}.
\end{align*}%
\setlength \abovedisplayskip{1ex plus0pt minus1pt}%

A state uniformly distributed in the informed subset, $\xfhat \sim \uniform{\fhatSet}$, can thus be calculated from a sample drawn uniformly from a unit $n$-ball, $\xball \sim \uniform{\ballSet}$, through a transformation \eqref{eqn:finalL}, rotation \eqref{eqn:svd}, and translation,
\begin{align} \label{eqn:sample}
	\xfhat = \mathbf{C}\mathbf{L}\xball + \xcentre.
\end{align}%
This procedure is presented algorithmically in Alg.~\ref{algo:sample}.

\begin{algorithm}[tb]%
\algorithmStyle
	\caption{Informed \acs{RRTstar}$\left(\xstart, \xgoal \right)$}\label{algo:body}
	$\vertexSet \gets \left\lbrace \xstart \right\rbrace$\;
	$\edgeSet \gets \emptyset$\;
	\newAlgoLine{$\solnSet \gets \emptyset$\;} \label{algo:body:init}
	$\treeGraph = \left( \vertexSet, \edgeSet \right)$\;
	\For{$\mathrm{iteration} = 1 \ldots N$}
	{
		\newAlgoLine{$\cbest \gets \min_{\xsoln\in\solnSet}\left\lbrace \mathtt{Cost}\left( \xsoln \right)\right\rbrace$\;} \label{algo:body:best}
		\newAlgoLine{$\xrand \gets \mathtt{Sample}\left(\xstart, \xgoal, \cbest \right)$\;} \label{algo:body:sample}
		$\xnearest \gets \mathtt{Nearest}\left(\treeGraph, \xrand \right)$\;
		$\xnew \gets \mathtt{Steer}\left(\xnearest, \xrand\right)$\;
		\If{$\mathtt{CollisionFree}\left(\xnearest, \xnew\right)$}
		{
			$\vertexSet \gets \cup \left\lbrace \xnew \right\rbrace$\;
			$\nearSet \gets \mathtt{Near}\left(\treeGraph, \xnew, r_{\mathrm{RRT}^*} \right)$\;
			$\xmin \gets \xnearest$\;
			$\cmin \gets \mathtt{Cost}\left( \xmin \right) + c\cdot\mathtt{Line}\left( \xnearest, \xnew \right)$\;
			\For{$\forall \xnear \in \nearSet$}	
			{
				$\cnew \gets \mathtt{Cost}\left( \xnear \right) + c\cdot\mathtt{Line}\left( \xnear, \xnew \right)$\;
				\If{$\cnew < \cmin$}
				{
					\If{$\mathtt{CollisionFree}\left( \xnear, \xnew \right)$}
					{
						$\xmin \gets \xnear$\;
						$c_{\rm min} \gets c_{\rm new}$\;
					}
				}
			}
			$\edgeSet \gets \edgeSet \cup \left\lbrace\left(\xmin, \xnew \right)\right\rbrace$\;
			
			\For{$\forall \xnear \in \nearSet$}
			{
				$\cnear \gets \mathtt{Cost}\left(\xnear \right)$\;
				$\cnew \gets \mathtt{Cost}\left( \xnew \right) + c\cdot\mathtt{Line}\left( \xnew, \xnear \right)$\;
				\If{$\cnew < \cnear$}
				{
					\If{$\mathtt{CollisionFree}\left( \xnew, \xnear \right)$}
					{
						$\xparent \gets \mathtt{Parent}\left(\xnear \right)$\;
						$\edgeSet \gets \edgeSet \setminus \left\lbrace \left(\xparent, \xnear \right) \right\rbrace$\;
						$\edgeSet \gets \edgeSet \cup \left\lbrace \left( \xnew, \xnear\right)\right\rbrace$\;
					}
				}
			}
			
			\newAlgoLine
			{
			\If{ $\mathtt{InGoalRegion}\left( \xnew \right)$ \label{algo:body:goalStart}}
			{
				$\solnSet \gets \solnSet \cup \left\lbrace \xnew \right\rbrace$\;\label{algo:body:goalEnd}
			}
			}
		}
	}
	\Return{$\treeGraph$}\;	
\end{algorithm}%
\section{Informed \acs{RRTstar}}\label{sec:algorithm}%
An example algorithm using direct informed sampling, Informed \acs{RRTstar}, is presented in Algs.~\ref{algo:body} and \ref{algo:sample}.
It is identical to \ac{RRTstar} as presented in \cite{karaman_ijrr11}, with the addition of lines \ref{algo:body:init}, \ref{algo:body:best}, \ref{algo:body:sample}, \ref{algo:body:goalStart}, and \ref{algo:body:goalEnd}.
Like \ac{RRTstar}, it searches for the optimal path, $\bestPath$, to a planning problem by incrementally building a tree in state space, $\treeGraph = \left( \vertexSet, \edgeSet \right)$,  consisting of a set of vertices, $\vertexSet \subseteq \freeSet$, and edges, $\edgeSet \subseteq \freeSet \times \freeSet$.
New vertices are added by growing the graph in free space towards randomly selected states.
The graph is rewired with each new vertex such that the cost of the nearby vertices are minimized.

The algorithm differs from \ac{RRTstar} in that once a solution is found, it focuses the search on the part of the planning problem that can improve the solution.
It does this through direct sampling of the ellipsoidal heuristic.
As solutions are found (line \ref{algo:body:goalStart}), Informed \acs{RRTstar} adds them to a list of possible solutions (line \ref{algo:body:goalEnd}).
It uses the minimum of this list (line \ref{algo:body:best}) to calculate and sample $\fhatSet$ directly (line \ref{algo:body:sample}).
As is conventional, we take the minimum of an empty set to be infinity.
The new subfunctions are described below, while descriptions of subfunctions common to \ac{RRTstar} can be found in \cite{karaman_ijrr11}:

\texttt{Sample}: Given two poses, $\xfrom,\,\xto \in \freeSet$ and a maximum heuristic value, $\cmax \in \mathbb{R}$, the function $\mathtt{Sample}\left( \xfrom, \xto, \cmax \right)$ returns \ac{iid} samples from the state space, $\xnew \in \stateSet$, such that the cost of an optimal path between $\xfrom$ and $\xto$ that is constrained to go through $\xnew$ is less than $\cmax$ as described in Section~\ref{sec:ellipse} and Alg.~\ref{algo:sample}. In most planning problems, $\xfrom \equiv \xstart$, $\xto \equiv \xgoal$, and lines \ref{algo:sample:start} to \ref{algo:sample:end} of Alg.~\ref{algo:sample} can be calculated once at the start of the problem.

\texttt{InGoalRegion}: Given a pose, $\statex \in \freeSet$, the function $\mathtt{InGoalRegion}\left(\statex \right)$ returns $\mathtt{True}$ if and only if the state is in the goal region, $\goalSet$, as defined by the planning problem, otherwise it returns $\mathtt{False}$.
One common goal region is a ball of radius $\radius_{\rm goal}$ centred about the goal, i.e.,
\begin{align*}%
	\goalSet = \left\lbrace \statex \in \freeSet \;\; \middle| \;\; \norm{\statex - \xgoal}{2} \leq \radius_{\rm goal} \right\rbrace.
\end{align*}%

\texttt{RotationToWorldFrame}: Given two poses as the focal points of a hyperellipsoid, $\xfrom,\,\xto \in \stateSet$, the function $\mathtt{RotationToWorldFrame}\left(\xfrom, \xto \right)$ returns the rotation matrix, $\mathbf{C} \in SO\left( n \right)$, from the hyperellipsoid-aligned frame to the world frame as per \eqref{eqn:svd}.
As previously discussed, in most planning problems this rotation matrix only needs to be calculated at the beginning of the problem.

\texttt{SampleUnitNBall}: The function, $\mathtt{SampleUnitNBall}$ returns a uniform sample from the volume of an $n$-ball of unit radius centred at the origin, i.e. $\xball \sim \uniform{\ballSet}$.

\begin{algorithm}[tb]%
\algorithmStyle
	\caption{$\mathtt{Sample}\left(\xstart, \xgoal, \cmax \right)$}\label{algo:sample}
	\If{$\cmax < \infty$}
	{
		$\cmin \gets \norm{\xgoal - \xstart}{2}$\; \label{algo:sample:start}
		$\xcentre \gets \left(\xstart + \xgoal\right)/2$\;
		$\mathbf{C} \gets \mathtt{RotationToWorldFrame\left(\xstart, \xgoal \right)}$\; \label{algo:sample:end}
		$\radius_{1} \gets \cmax/2$\;
		$\left\lbrace \radius_i\right\rbrace_{i = 2,\ldots,n} \gets \left(\sqrt{\cmax^2 - \cmin^2}\right)/2$\;
		$\mathbf{L} \gets \diag\left\lbrace\radius_1, \radius_2, \ldots, \radius_n\right\rbrace$\;
		$\xball \gets \mathtt{SampleUnitNBall}$\;
		$\xrand \gets \left( \mathbf{C}\mathbf{L}\xball + \xcentre\right) \cap \stateSet$\;
	}
	\Else
	{
		$\xrand \sim \uniform{\stateSet}$\;
	}
	\Return{$\xrand$};
\end{algorithm}%
\subsection{Calculating the Rewiring Radius}\label{sec:algorithm:rewire}
At each iteration, the rewiring radius, $r_{\mathrm{RRT}^*}$, must be large enough to guarantee almost-sure asymptotic convergence while being small enough to only generate a tractable number of rewiring candidates.
Karaman and Frazzoli \cite{karaman_ijrr11} present a lower-bound for this rewiring radius in terms of the measure of the problem space and the number of vertices in the graph.
Their expression assumes a uniform distribution of samples of a unit square.
As Informed \acs{RRTstar} uniformly samples the \emph{subset} of the planning problem that can improve the solution, a rewiring radius can be calculated from the measure of this informed subset and the related vertices inside it.
This updated radius reduces the amount of rewiring necessary and further improves the performance of Informed \ac{RRTstar}.
Ongoing work is focused on finding the exact form of this expression, but the radius provided by \cite{karaman_ijrr11} appears appropriate.
There also exists a $k$-nearest neighbour version of this expression.

\section{Simulations}\label{sec:sim}
Informed \acs{RRTstar} was compared to \ac{RRTstar} on a variety of simple planning problems (Figs.~\ref{fig:probDefn} to \ref{fig:gapExample}) and randomly generated worlds (e.g., Figs.~\ref{fig:randomWorld2},~\ref{fig:randomWorld}).
Simple problems were used to test specific challenges, while the random worlds were used to provide more challenging problems in a variety of state dimensions.

Fig.~\ref{fig:probDefn}(a) was used to examine the effects of the problem range and the ability to find paths within a specified tolerance of the true optimum, with the width of the obstacle, $w$, selected randomly.
Fig.~\ref{fig:probDefn}(b) was used to demonstrate Informed \acs{RRTstar}'s ability to find topologically distinct solutions, with the position of the narrow passage, $y_g$, selected randomly.
For these toy problems, experiments were ended when the planner found a solution cost within the target tolerance of the optimum.
Random worlds, as in Fig.~\ref{fig:randomWorld}, were used to test Informed \acs{RRTstar} on more complicated problems and in higher state dimensions by giving the algorithms $60$ seconds to improve their initial solutions.
For each variation of every experiment, $100$ different runs of both \ac{RRTstar} and Informed \acs{RRTstar} were performed with a common pseudo-random seed and map.

The algorithms share the same unoptimized code, allowing for the comparison of relative computational time\footnotemark{}.
While further optimization would reduce the effect of graph size on the computational cost and reduce the difference between the two planners, as they have approximately the same cost per iteration it will not effect the order.
To minimize the effects of the steer parameter on our results, we set it equal to the \ac{RRTstar} rewiring radius at each iteration calculated from $\gamma_{\rm RRT} = 1.1\gamma_{\rm RRT}^*$, a choice we found improved the performance of \ac{RRTstar}.
As discussed in Section~\ref{sec:algorithm:rewire}, for Informed \acs{RRTstar} we calculated the rewiring radius for the subproblem defined by the current solution using the expression in \cite{karaman_ijrr11}.
\footnotetext{Experiments were run in Ubuntu 12.04 on an Intel i5-2500K CPU with 8GB of RAM.}

\begin{figure}[tb]%
	\centering
	\includegraphics[scale=1]{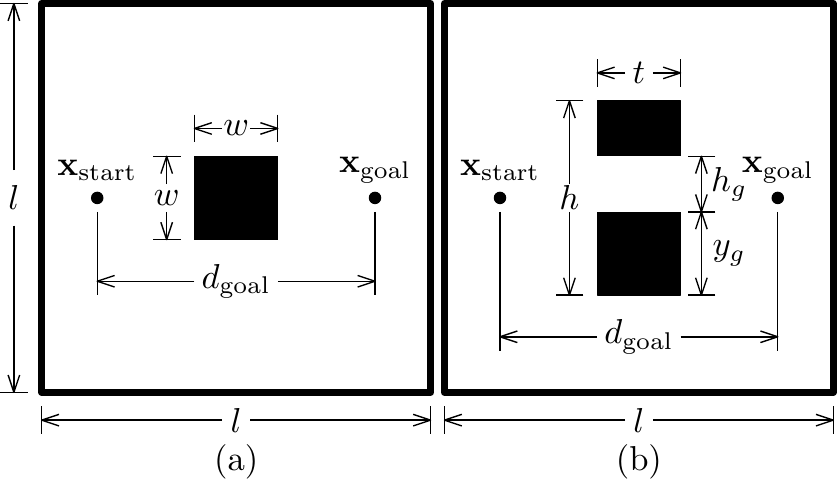}
	\caption{The two planning problems used in Section~\ref{sec:sim}. The width of the obstacle, $w$, and the location of the gap, $y_g$, were selected randomly for each experimental run.}
	\label{fig:probDefn}
\end{figure}%

Experiments varying the width of the problem range, $l$, while keeping a fixed distance between the start and goal show that Informed \acs{RRTstar} finds a suitable solution in approximately the same time regardless of the relative size of the problem (Fig.~\ref{fig:mapStudy}).
As a result of considering only the informed subset once an initial solution is found, the size of the search space is independent of the planning range (Fig.~\ref{fig:convergeExample}).
In contrast, the time needed by \ac{RRTstar} to find a similar solution increases as the problem range grows as proportionately more time is spent searching states that cannot improve the solution (Fig.~\ref{fig:mapStudy}).

Experiments varying the target solution cost show that Informed \acs{RRTstar} is capable of finding near-optimal solutions in significantly fewer iterations than \ac{RRTstar} (Fig.~\ref{fig:convergenceStudy}).
The direct sampling of the informed subset increases density around the optimal solution faster than global sampling and therefore increases the probability of improving the solution and further focusing the search.
In contrast, \ac{RRTstar} has uniform density across the entire planning domain and improving the solution actually \emph{decreases} the probability of finding further improvements  (Fig.~\ref{fig:convergeExample}).

Experiments varying the height of $h_g$ in Fig.~\ref{fig:probDefn}(b) demonstrate that Informed \acs{RRTstar} finds difficult passages that improve the current solution, regardless of their homotopy class, quicker than \ac{RRTstar} (Fig.~\ref{fig:gapStudy}).
Once again, the result of considering only the informed subset is an increased state density in the region of the planning problem that includes the optimal solution.
Compared to global sampling, this increases the probability of sampling within difficult passages, such as narrow gaps between obstacles, decreasing the time necessary to find such solutions (Fig.~\ref{fig:gapExample}).

Finally, experiments on random worlds demonstrate that the improvements of Informed \acs{RRTstar} apply to a wide range of planning problems and state dimensions (Fig.~\ref{fig:dimensionStudy}).

\begin{figure}[t]%
	\centering
	\includegraphics[width=\columnwidth]{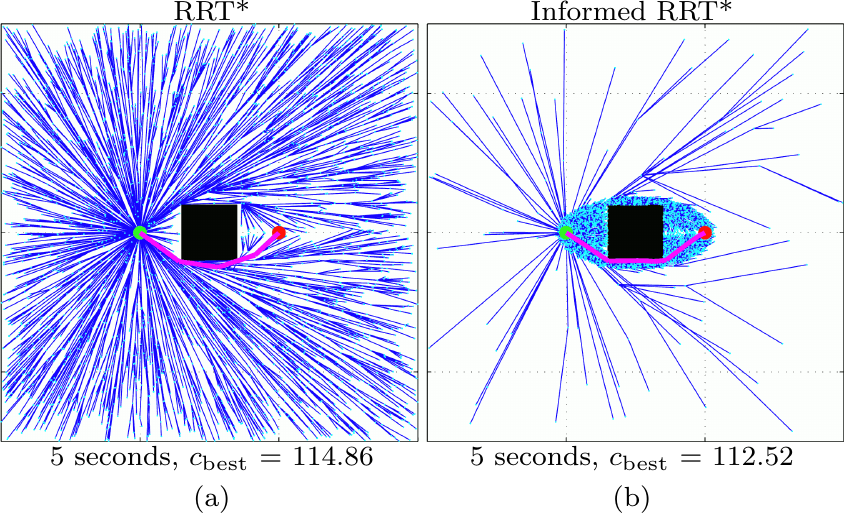}
	\caption{An example of Fig.~\ref{fig:probDefn}(a) after $5$ seconds for a problem with an optimal solution cost of $112.01$. Note that the presence of an obstacle provides a lower bound on the size of the ellipsoidal subset but that Informed \acs{RRTstar} still searches a significantly reduced domain than \acs{RRTstar}, increasing both the convergence rate and quality of final solution.}
	\label{fig:convergeExample}
\end{figure}%
\begin{figure}[t]%
	\centering
	\includegraphics[width=\columnwidth]{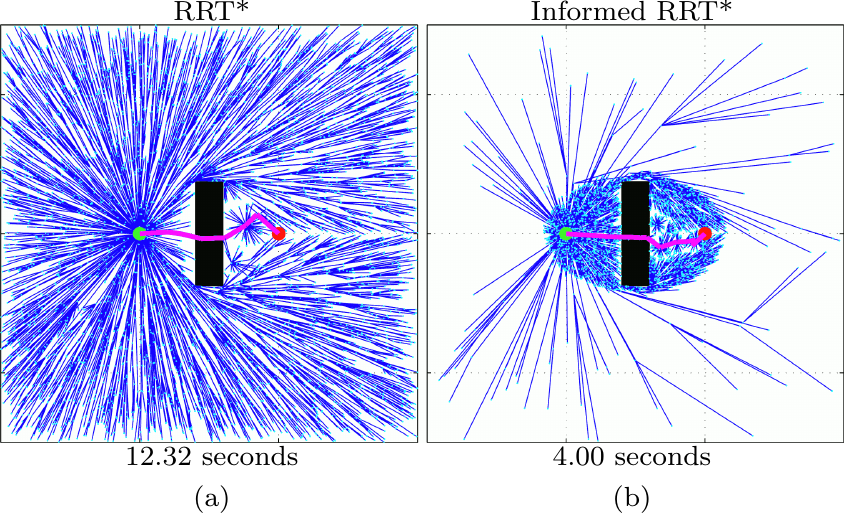}
	\caption{An example of Fig.~\ref{fig:probDefn}(b) for a $3\%$ off-centre gap. By focusing the search space on the subset of states that may improve an initial solution flanking the obstacle, Informed \acs{RRTstar} is able to find a path through the narrow opening in $4.00$ seconds while \acs{RRTstar} requires $12.32$ seconds.}
	\label{fig:gapExample}
\end{figure}%
\begin{figure}[t]%
	\centering
	\includegraphics[width=\columnwidth]{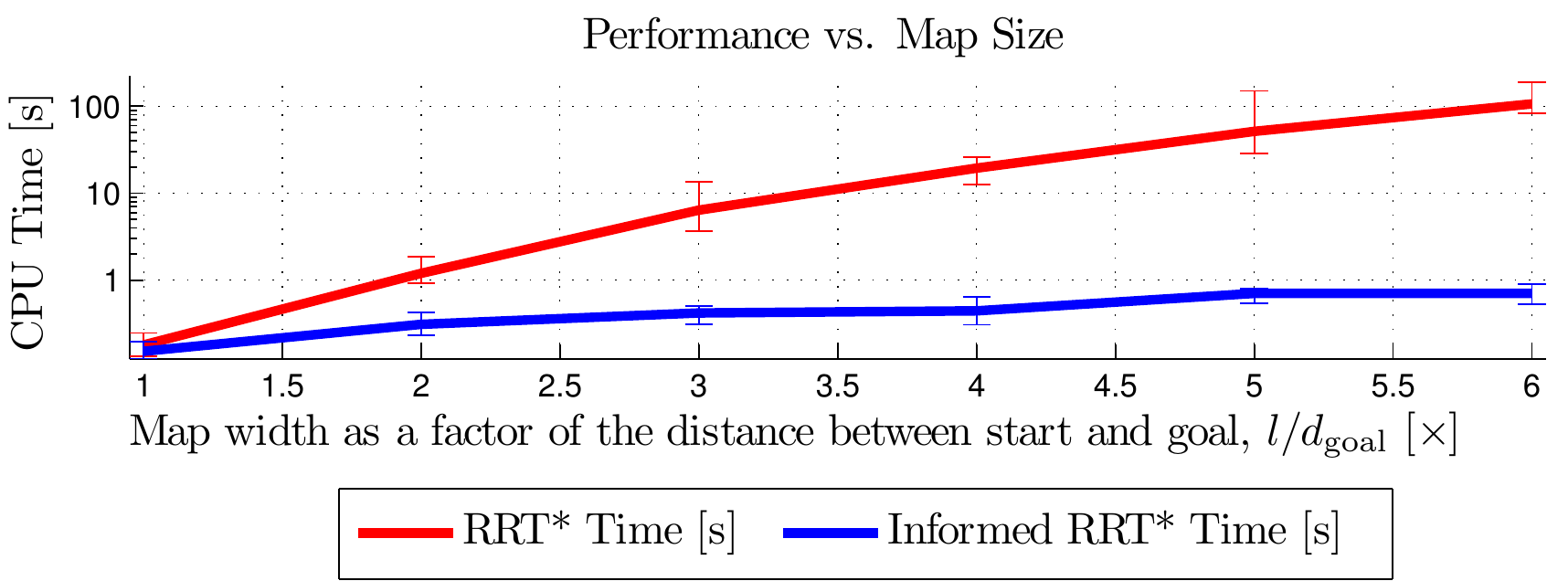}
	\caption{The median computational time needed by \acs{RRTstar} and Informed \acs{RRTstar} to find a path within 2\% of the optimal cost in $\mathbb{R}^2$ for various map widths, $l$, for the problem in Fig.~\ref{fig:probDefn}(a). Error bars denote a nonparametric $95\%$ confidence interval for the median number of iterations calculated from $100$ independent runs.}
	\label{fig:mapStudy}
\end{figure}%
\begin{figure}[t]%
	\centering
	\includegraphics[width=\columnwidth]{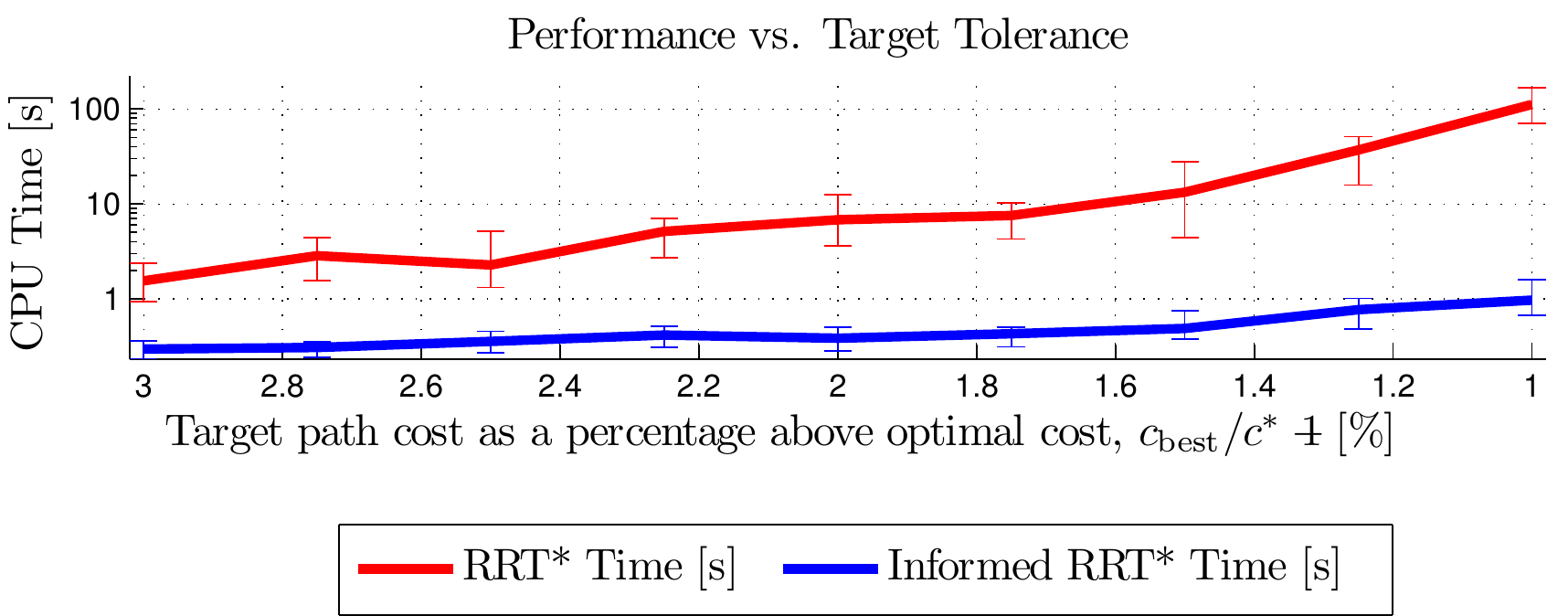}
	\caption{The median computational time needed by \acs{RRTstar} and Informed \acs{RRTstar} to find a path within the specified tolerance of the optimal cost, $\cideal$, in $\mathbb{R}^2$ for the problem in Fig.~\ref{fig:probDefn}(a). Error bars denote a nonparametric $95\%$ confidence interval for the median number of iterations calculated from $100$ independent runs.}
	\label{fig:convergenceStudy}
\end{figure}%
\begin{figure}[t]%
	\centering
	\includegraphics[width=\columnwidth]{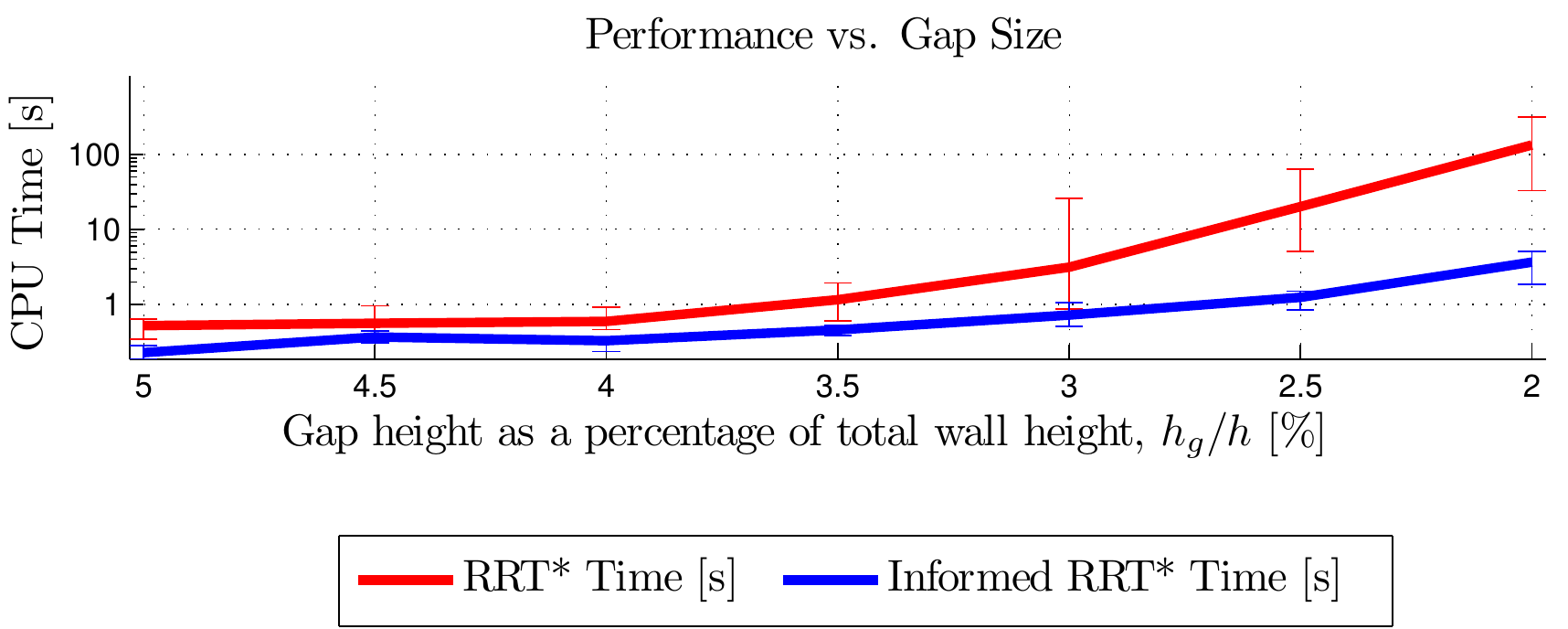}
	\caption{The median computational time needed by \acs{RRTstar} and Informed \acs{RRTstar} to find a path cheaper than flanking the obstacle for various gap ratios, $h_g/h$ for the problem defined in Fig.~\ref{fig:probDefn}(b). Error bars denote a nonparametric $95\%$ confidence interval for the median number of iterations calculated from $100$ independent runs.}
	\label{fig:gapStudy}
\end{figure}%
\begin{figure}[t]%
	\centering
	\includegraphics[width=\columnwidth]{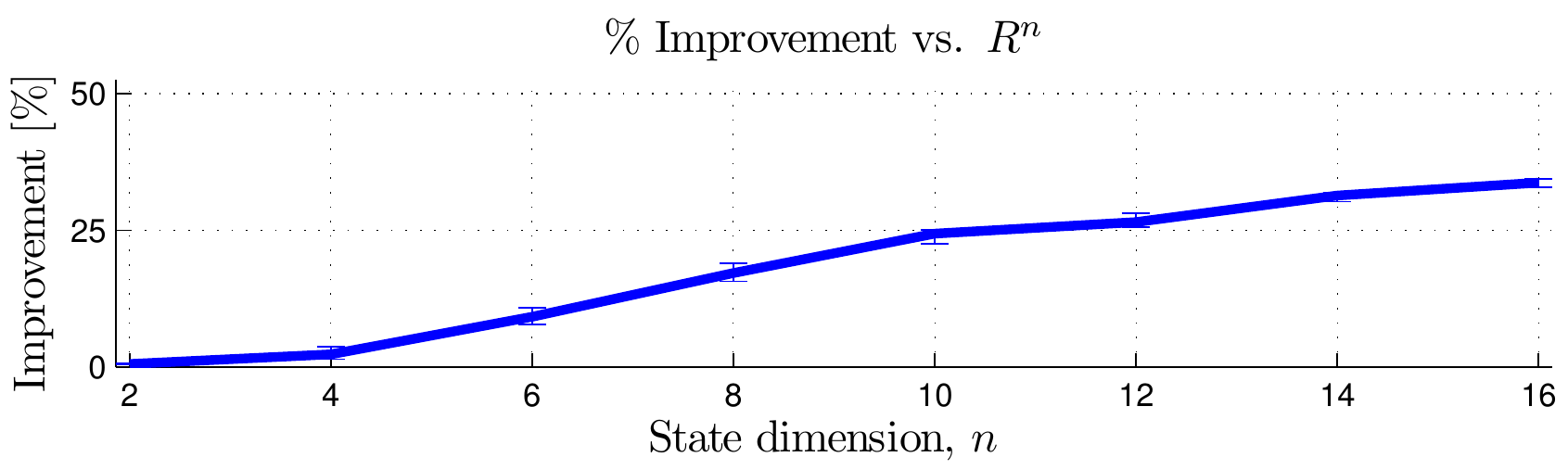}
	\caption{The median performance of \acs{RRTstar} and Informed \acs{RRTstar} $60$ seconds after finding an initial solution for random worlds (e.g., Figs.~\ref{fig:randomWorld2},~\ref{fig:randomWorld}) in $\mathbb{R}^n$. Plotted as the relative difference in cost, $(\cbest^{\mbox{\tiny RRT*}} - \cbest^{\mbox{\tiny Informed RRT*}})/(\cbest^{\mbox{\tiny RRT*}})$. Error bars denote a nonparametric $95\%$ confidence interval for the median number of iterations calculated from $100$ independent runs.}
	\label{fig:dimensionStudy}
\end{figure}%
\section{Discussion \& Conclusion}\label{sec:end}
In this paper, we discuss that a necessary condition for \ac{RRTstar} algorithms to improve a solution is the addition of a state from a subset of the planning problem, $\fSet \subseteq \stateSet$.
For problems seeking to minimize path length in $\mathbb{R}^n$, this subset can be estimated, $\fhatSet \supseteq \fSet$, by a prolate hyperspheroid (a special type of hyperellipsoid) with the initial and goal states as focal points.
It is shown that the probability of adding a new state from this subset through rejection sampling of a larger set becomes arbitrarily small as the dimension of the problem increases, the size of the sampled set increases, or the solution approaches the theoretical minimum.
A simple method to sample $\fhatSet$ directly is presented that allows for the creation of informed-sampling planners, such as Informed \acs{RRTstar}.
It is shown that Informed \acs{RRTstar} outperforms \ac{RRTstar} in the ability to find near-optimal solutions in finite time regardless of state dimension without requiring any assumptions about the optimal homotopy class.

Informed \acs{RRTstar} uses heuristics to shrink the planning problem to subsets of the original domain.
This makes it inherently dependent on the current solution cost, as it cannot focus the search when the associated prolate hyperspheroid is larger than the planning problem itself.
Similarly, it can only shrink the subset down to the lower bound defined by the optimal solution.
We are currently investigating techniques to focus the search without requiring an initial solution.
These techniques, such as \ac{BITstar} \cite{gammell_arxiv14c}, incrementally \emph{increase} the search subset.
By doing so, they prioritize the initial search of low-cost solutions.

An \ac{OMPL} implementation of Informed \acs{RRTstar} is described at \href{http://asrl.utias.utoronto.ca/code}{\footnotesize\url{http://asrl.utias.utoronto.ca/code}}.

\section*{Acknowledgment}%
This research was funded by contributions from the \ac{NSERC} through the \ac{NCFRN}, the Ontario Ministry of Research and Innovation's Early Researcher Award Program, and the \ac{ONR} Young Investigator Program.
\end{spacing}%
\end{document}